\documentclass[lettersize,journal]{IEEEtran}  

\usepackage{amsthm}

\usepackage{amsthm}
\usepackage{times}
\usepackage{multicol}
\usepackage[bookmarks=true]{hyperref}
\usepackage{xcolor}
\usepackage{hyperref}
\usepackage{amsmath, amssymb}
\usepackage{amsfonts}
\usepackage{graphicx}
\usepackage{siunitx}
\usepackage{standalone}
\usepackage{booktabs}
\usepackage[ruled,vlined,linesnumbered,noend]{algorithm2e}
\usepackage{mdframed}
\usepackage{fancyvrb,multirow}
\usepackage{soul}
\usepackage{dsfont,mathabx}
\usepackage{array, booktabs}
\usepackage[para,online,flushleft]{threeparttable}

\usepackage{subfigure}
\usepackage{booktabs}
\usepackage{makecell}
\usepackage{tabularx}

\usepackage{rotating} 
\usepackage{booktabs} 
\usepackage{multirow} 
\usepackage{array}    
\usepackage{graphicx} 

\newtheorem{theorem}{Theorem}

\newtheorem{lemma}[theorem]{Lemma}
\theoremstyle{definition}

\title{\LARGE \bf
  CoCoPlan: Adaptive Coordination and Communication for Multi-robot Systems
  in Dynamic and Unknown Environments
}

\author{Xintong Zhang$^{2, \dagger}$, Junfeng Chen$^{1, \dagger}$, Yuxiao Zhu$^2$, Bing Luo$^2$,
and Meng Guo$^1$
\thanks{Received 29 August 2025; Revised 29 November 2025; Accepted 29 December 2025.
  This article was recommended for publication by Editor M. Ani Hsieh upon evaluation
  of the Associate Editor and Reviewers' comments.}
\thanks{This work was supported by the National Natural Science Foundation
of China (NSFC) under grants U2241214 and T2121002.}%
\thanks{$^\dagger$Equal contribution.
  The authors are with $^1$the School of Advanced Manufacturing and Robotics,
  Peking University, Beijing 100871, China;
and $^2$the Division of Natural and Applied Sciences, Duke Kunshan University, Suzhou 215316, China.}
\thanks{Corresponding author: Meng Guo, \texttt{meng.guo@pku.edu.cn}.}
\thanks{Digital Object Identifier (DOI): see top of this page.}
}

\begin{document}
\maketitle
 \thispagestyle{empty}
\pagestyle{empty}

\begin{abstract}
  Multi-robot systems can greatly enhance efficiency through coordination and collaboration,
  yet in practice,
full-time communication is rarely available and interactions are constrained to close-range
exchanges. Existing methods either maintain all-time connectivity, rely on fixed schedules,
or adopt pairwise protocols, but none adapt effectively to dynamic spatio-temporal task
distributions under limited communication, resulting in suboptimal coordination. To address
this gap, we propose {CoCoPlan}, a unified framework that co-optimizes collaborative
task planning and team-wise intermittent communication. Our approach integrates a
branch-and-bound architecture that jointly encodes task assignments and communication
events, an adaptive objective function that balances task efficiency against communication
latency, and a communication event optimization module that strategically determines when,
where and how the global connectivity should be re-established.
Extensive experiments demonstrate that it outperforms state-of-the-art methods
by achieving a {22.4\%} higher task completion rate,
reducing communication overhead by \textbf{58.6\%},
and improving the scalability by supporting up to \textbf{100 robots}
in dynamic environments.
Hardware experiments include the complex 2D office environment and large-scale 3D disaster-response scenario.
\end{abstract}
\begin{IEEEkeywords}
    Multi-robot System,
    Task Planning,
    Adaptive Communication.
\end{IEEEkeywords}
\section{Introduction}\label{sec:intro}
Multi-agent systems improve efficiency through concurrent execution and collaboration,
particularly in complex scenarios. Traditional approaches focus either on spatial task planning~\cite{huang2025novel},
or on collective control such as formation maintenance~\cite{ren2011distributed}, while assuming
unrestricted communication for instantaneous information sharing. This assumption breaks down
in real deployments, where communication is limited by obstacles~\cite{yu2021resilient}, or range
constraints~\cite{da2024communication}.
Such limitations necessitate coordination of communication events. The challenge intensifies in dynamic unknown
environments with unpredictable tasks and workspace models. Existing methods remain limited:
static schedules~\cite{huang2025novel} fail under unexpected tasks; dynamic
approaches~\cite{buyukkocak2021planning} require prior knowledge of distributions; and pairwise
protocols~\cite{aragues2020intermittent} lack team-wide coordination, leading to degraded
information flow. This creates a critical gap in jointly handling unknown spatiotemporal task
distributions and team-wide intermittent communication under constraints.
Three fundamental challenges arise: the {online discovery}
demands rapid replanning of assignments and communication when tasks emerge~\cite{chen2025real};
the {sparse connectivity management} requires
strategic use of limited communication windows to preserve coordination~\cite{saboia2022achord};
the temporal constraints among tasks must be maintained between task deadlines and intermittent
communication opportunities~\cite{guo2018multirobot}.
These challenges amplify with large fleets,
where co-optimizing tasks and communication becomes combinatorially complex.

\begin{figure}[!t]
    \centering
    \includegraphics[width=0.9\linewidth, height=0.85\linewidth]{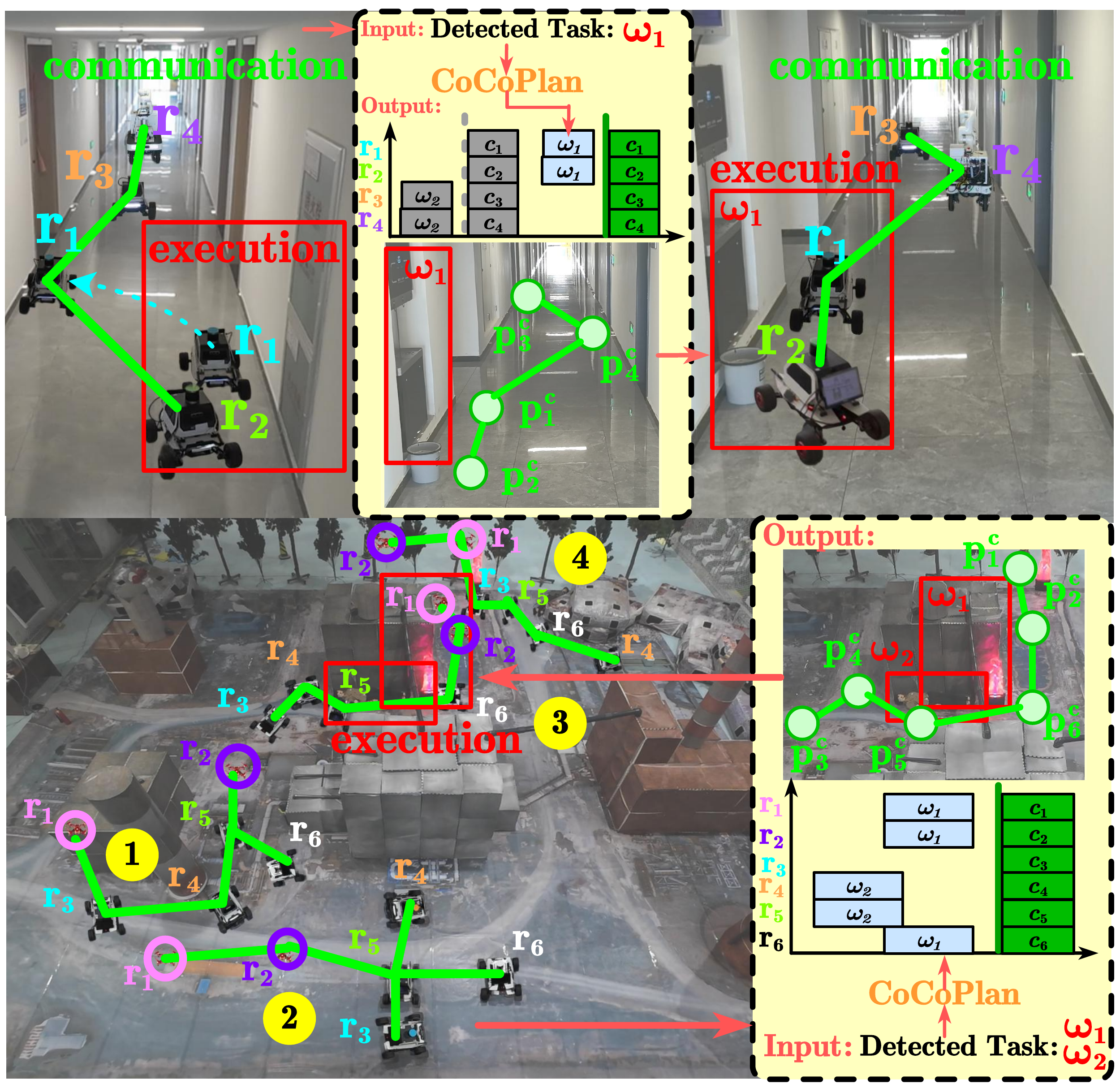}
    \vspace{-2mm}
    \caption{
      Hardware demonstration of the proposed adaptive coordination and communication scheme
        ($r_i$ denotes an individual hardware agent $i$, $\forall i \in \mathcal{N}$):
      $4$ UGVs to fulfil office errands such as delivery and cleaning as requested online (\textbf{top});
      $2$ UAVs and $4$ UGVs to detect fire hazards, search and rescue victims
      during a disaster-response mission (\textbf{bottom}).
    }
    \label{fig:first}
    \vspace{-5mm}
  \end{figure}

\vspace{-0.8em} 

\subsection{Related work}

\begin{figure*}[!t]
    \centering
    \includegraphics[width=0.9\linewidth,height=0.25\linewidth]{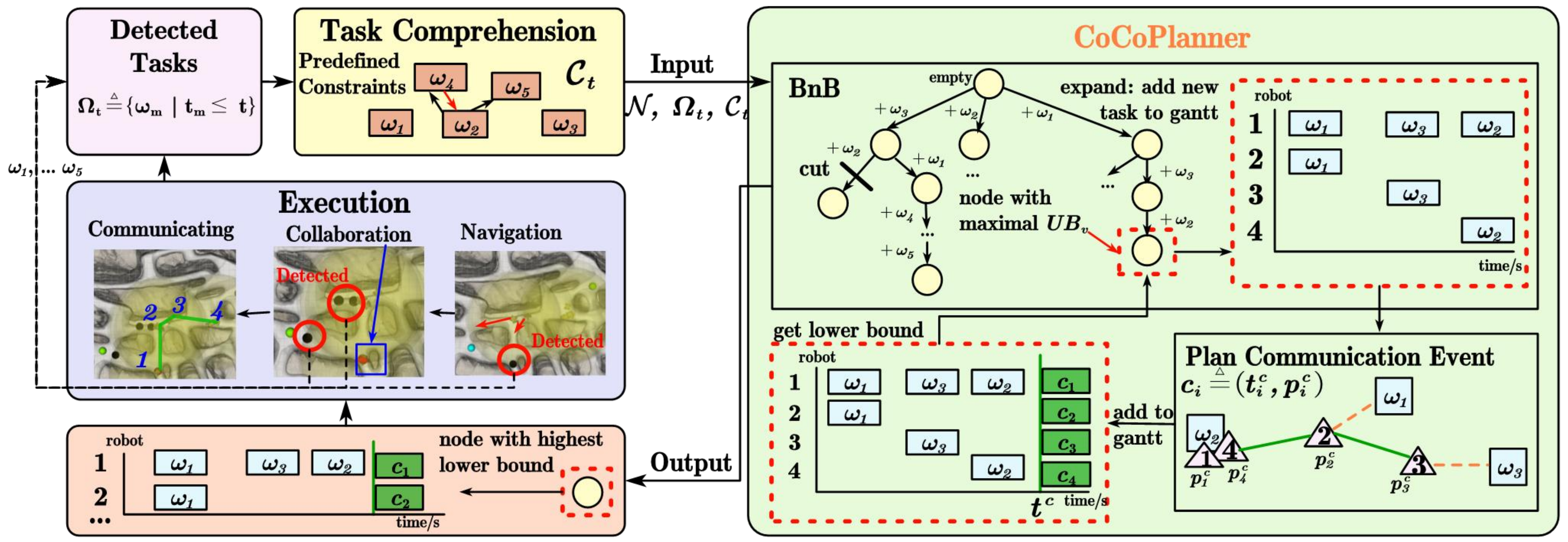}
    \vspace{-2mm}
    \caption{Overview of the \textbf{CoCoPlan} framework,
      which integrates real-time task detection,
      a branch-and-bound planner for joint task and communication scheduling,
      and an adaptation scheme to unknown spatio-temporal task distributions.}
    \label{fig:framework}
    \vspace{-4mm}
  \end{figure*}

\subsubsection{Coordination under Limited Communication}
Communication is critical for multi-agent systems, especially in scenarios such as exploration and inspection where agents
must exchange information to maintain alignment and improve outcomes. Existing approaches
either employ an end-to-end connectivity model with constant communication, or an intermittent
model where agents communicate only when necessary. Both involve trade-offs in efficiency,
flexibility, and adaptability. All-time communication usually focuses on connectivity control,
often using graph theory~\cite{griparic2022consensus, nguyen2024connectivity}, either by maintaining all initial links~\cite{griparic2022consensus}
or dynamically adding and removing them while preserving connectivity~\cite{derbakova2011decentralized}.
While this approach guarantees reliable data transfer, 
it often overlooks the communication needs of high-level, collaborative tasks. 
Moreover, it can be impractical in the presence of wireless channel uncertainties.
Some studies address the joint constraints of connectivity
and temporal tasks~\cite{luo2019minimum}, yet they
typically impose strict proximity requirements, reducing overall efficiency in large-scale deployments.

\subsubsection{Intermittent Communication under Complex Temporal Tasks}
To overcome these limitations, intermittent communication has been proposed as an alternative
solution~\cite{aragues2020intermittent, guo2018multirobot, kantaros2019temporal}. In this model, the communication network can be
temporarily interrupted, and agents exchange information only at certain times.
To this end, several related works adopt intermittent communication
by predefining fixed intervals and optimizing meeting locations accordingly,
including centralized~\cite{wang2023multi}, distributed~\cite{kantaros2016distributed}
and predetermined meeting-point approaches~\cite{da2024communication}.
However, these works typically fix communication times while optimizing locations,
lacking simultaneous optimization of both aspects.
Recent work mainly studies pair-wise intermittent communication~\cite{aragues2020intermittent,
guo2018multirobot, kantaros2019temporal}, where communication occurs only when necessary. While
this improves efficiency, it enforces rigid patterns such as fixed schedules~\cite{aragues2020intermittent,
kantaros2019temporal} or predetermined meeting points~\cite{guo2018multirobot}, and introduces
delays due to asynchronous execution, which is problematic for tasks that need timely coordination. Thus,
all-time communication sacrifices task efficiency, while pairwise protocols slow
information propagation. To mitigate these drawbacks, team-wise intermittent communication was proposed,
in which all agents periodically re-establish global connectivity~\cite{saboia2022achord}.
This differs from sub-team or clustered communication~\cite{ren2023matcpf, yang2024teamwise},
where ``team-wise'' denotes coordination within sub-teams or dynamically clustered groups.
Nevertheless, existing team-wise approaches still assume known environments
and task distributions, limiting applicability in uncertain scenarios.

\vspace{-0.9em} 

\subsection{Our Method}
In this work,
a multi-agent system handles complex temporally constrained tasks continuously released online as shown in Fig.~\ref{fig:first}.
We formulate an online co-optimization problem to synthesize the collaborative task execution
and communication strategies simultaneously.
Unlike periodic team-wise or intermittent pair-wise communication protocols,
our approach dynamically optimizes \emph{when}, \emph{where} and \emph{how}
agents communicate or collaborate
based on real-time task specifications gathered online.
To achieve this, we develop a branch-and-bound (BnB) framework with
joint encoding of task planning and communication events within unified search nodes.
A novel objective function balances the execution efficiency against communication latency
for robustness across spatio-temporal task distributions.
Moreover, communication events are optimized via an iterative algorithm
that strategically determines the communication locations and timings.
As shown in Fig.~\ref{fig:first}, the method is validated through large-scale simulations
and hardware experiments against several strong baselines.

The main contributions are three-fold:
(I) A general online co-design of task planning and team-wise communication strategies,
ensuring the timely information propagation and efficient task fulfillment;
(II) An adaptive coordination mechanism via a novel objective function robust
to unknown task distributions and dynamic environments;
(III) Extensive validation in simulations and hardware experiments
demonstrating the effectiveness in complex real-world scenarios.


\section{Problem Description}\label{sec:problem}

\subsection{Workspace and Robot Model}
\label{subsec:env-task-model}

Consider a team of~$N$ agents that share the common
workspace~$\mathcal{W}\subset \mathbb{R}^2$.
Each agent~$i\in \mathcal{N}\triangleq \{1,\cdots,N\}$ is described
by its position~$p_i\in \mathcal{W}_i$,
sensor range~$d_i > 0$,
and the action~$a_i\in \mathcal{A}_i$,
where~$\mathcal{W}_i\subseteq \mathcal{W}$ is the allowed workspace;
and~$\mathcal{A}_i$ is the set of primitive actions.
More specifically, the actions $\mathcal{A}_i$ require different
capabilities from agent~$i$. For instance, repair and rescue actions necessitate
fine-grained manipulation capabilities,
while delivery requires transportation capacity.
The motion of each agent is characterized by its velocity~$v_i\in \mathbb{R}^2$ with a upper limit $\|v_i\| \leq v^{\texttt{max}}_i$,
note that $\|v_i\|$ can be adjusted during navigation.
Thus, the local plan of an agent is given by a sequence of
timed goal positions and performed actions,
i.e.,~$\tau_i\triangleq (t^1_i,\,g^1_i,\,a^1_i)
(t^2_i,\,g^2_i,\,a^2_i)\cdots$
with~$t^\ell_i\geq 0$,~$g^\ell_i \in \mathcal{W}_i$,~$a_i^\ell \in \mathcal{A}_i$ being
the time instant, goal position and action,
$\forall \ell\geq 1$.
In other words, under this local plan,
agent~$i\in \mathcal{N}$ should navigate to~$g_i^\ell$ with velocity $v_i$
and start performing~$a_i^\ell$ from time~$t_i^\ell$, for all~$\ell\geq 1$.
The collective plan for all agents is denoted by~$J \triangleq \{\tau_i\}, \forall i \in \mathcal{N}$.

\vspace{-0.8em} 
\subsection{Task Specification with Temporal Constraints}
\label{subsec: task-spec}

Collaborative tasks are dynamically released with unknown and possibly time-varying spatio-temporal distributions,
each defined as~$\omega_m \triangleq (S_m, \eta_m, \{(n_j, a_j)\}_{j=1}^{J_m})$ where
and each pair $\{(n_j, a_j)\}$ indicates that $n_j$ agents are required to perform action $a_j$ within $S_m$.
A task is detected only when $\textbf{min}_{i \in \mathcal{N}} \|p_i(t)-S_m\| \leq d_i$ and no obstacle between $p_i(t)$ and $S_m$ at $t_m$,
where~$\|\cdot\|$ is the Euclidean distance.
The detected tasks will later become
known to all robots during the team-wise communication process.
The set of detected tasks by time $t$ is $\boldsymbol{\Omega}_t \triangleq \{ \omega_m \mid t_m \leq t \}$.
For each task $\omega_r$, its execution begins at $t_r^\texttt{s} > 0$ and finishes at $t_r^\texttt{f} >0$.
Three types of temporal constraints are predefined between tasks $\omega_p$ and $\omega_q$:
  (1) precedence ($\omega_p \preceq \omega_q$) holds if $t_p^\texttt{f} \leq t_q^\texttt{s}$;
  (2) mutual exclusion ($\omega_p \parallel \omega_q$) holds if their execution intervals are disjoint but no specific ordering is imposed, i.e., $[t_p^\texttt{s}, t_p^\texttt{f}] \cap [t_q^\texttt{s}, t_q^\texttt{f}] = \emptyset$;
  (3) concurrency ($\omega_p \sim \omega_q$) holds if parallel execution with overlapping time intervals is permitted, i.e., $[t_p^\texttt{s}, t_p^\texttt{f}] \cap [t_q^\texttt{s}, t_q^\texttt{f}] \neq \emptyset$.
The set $\mathcal{P}_t$ contains all predefined temporal relations between pairs of tasks in $\boldsymbol{\Omega}_t$, i.e.,
\begin{equation}\label{eq:task-constraints}
\mathcal{P}_t \subseteq \big\{
    (\omega_i, \omega_j, \kappa)
    \, |\,
    \omega_i, \omega_j \in \boldsymbol{\Omega}_t,
    \kappa \in \{\preceq, \parallel, \sim\}
\big\},
\end{equation}
which can only be known at time~$t\geq 0$.
A task $\omega_m$ completes at finish time $t_m^\texttt{f}$ if two conditions hold:
(I) at time $t_m^\texttt{s}$, agents arrive at $S_m$ and for each required action $(n_j, a_j)$,
a team of $n_j$ agents continuously performs $a_j$ within $S_m$ for duration $\eta_m$;
and (II) the execution interval $[t_m^\texttt{s}, t_m^\texttt{f}]$
satisfies all relevant constraints in $\mathcal{P}_{t}$
involving $\omega_m$.


\begin{figure}[!t]
    \centering
    \includegraphics[width=0.9\linewidth,height=0.35\linewidth]{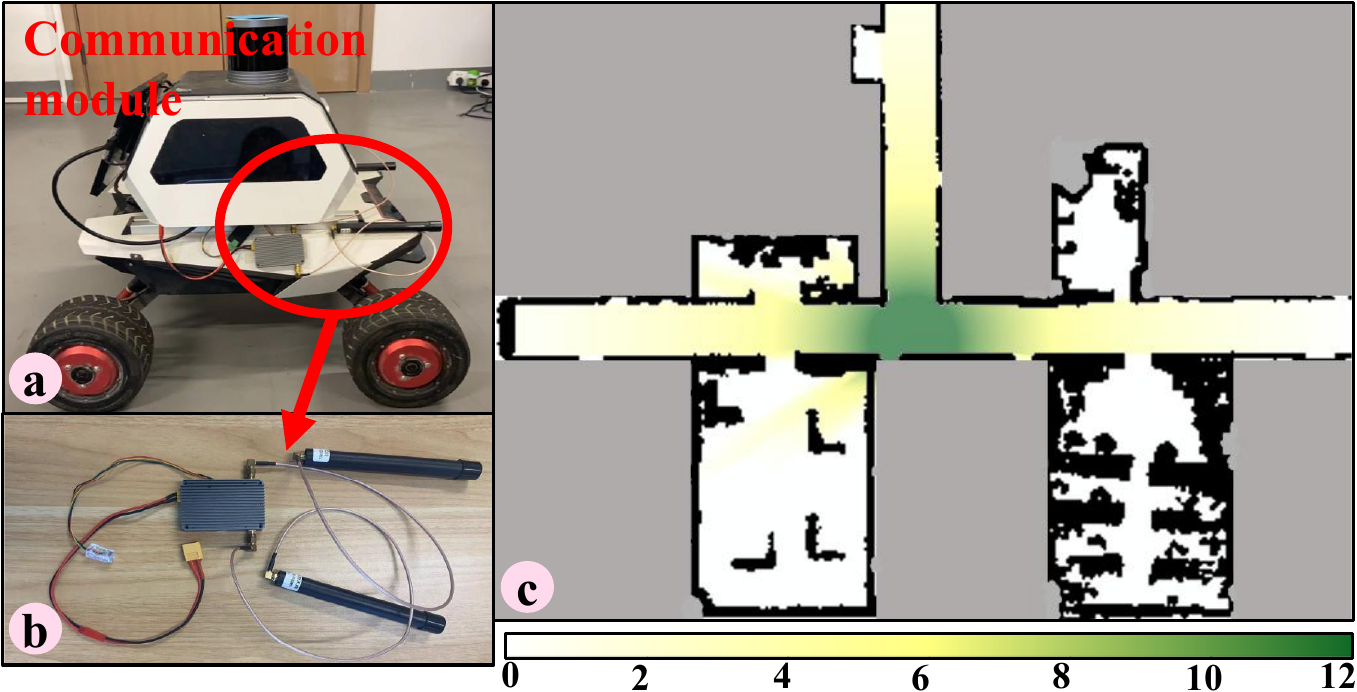}
    \vspace{-2mm}
    \caption{
      (\textbf{Left}): Ad-hoc local communication devices used in the hardware experiments;
      (\textbf{Right}): Map of the predicted communication quality used in the coordination
      of communication events.
    }
    \label{fig:communication-constraints}
    \vspace{-4mm}
  \end{figure}
\vspace{-0.8em}
\subsection{Communication Constraints}
\label{subsec:comm-const}
Agents cannot maintain the continuous real-time communication due to physical limitations.
As illustrated in Fig.~\ref{fig:communication-constraints}, successful communication
  between agents $i, j \in \mathcal{N}$ requires sufficient communication quality $Q_{ij}(t) > \delta$,
  where $\delta > 0$ is the minimum quality threshold. The communication quality $Q_{ij}(t)$ is defined as:
  $Q_{ij}(t) \triangleq P_t - PL(d_{ij}(t)) - \alpha\, d_{ij}^{\mathrm{obs}}(t)$,
  for locations $p_i, p_j \in \mathcal{W}_i, \mathcal{W}_j$,
where $P_t$ denotes the transmission power in dB,
and $PL(d_{ij}(t)) \triangleq PL(d_0) + 10 n \log_{10}\left(\frac{d_{ij}(t)}{d_0}\right)$ represents the free-space path loss,
with $d_0$ as the reference distance, $n$ as the path loss exponent,
and $d_{ij}(t) \triangleq \|p_i(t) - p_j(t)\|$ as the distance between agents $i$ and $j$ at time $t$.
Note that $d_{ij}^{\mathrm{obs}}(t)$ denotes the total length of obstacles
along the straight line connecting $p_i(t)$ and $p_j(t)$,
with $\alpha>0$ as the attenuation coefficient per meter of obstacle.
This constrained connectivity is modeled as a time-varying graph $G(t) \triangleq (\mathcal{N},\, \mathcal{E}(t))$
with the edge set
$\mathcal{E}(t) \triangleq \left\{ e_{ij} \,|\, Q_{ij}(t) > \delta \right\}$,
where~$e_{ij}$ is the communication link between agents $i$ and $j$.
The agent team $\mathcal{N}$ achieves the global communication at the time $t > 0$
if $G(t)$ is connected, namely:
$G(t) \in \mathcal{G}_{\texttt{c}}$,
where~$\mathcal{G}_{\texttt{c}}$ is the set of connected graphs.
We refer to this as team-wise intermittent communication,
requiring global connectivity across all agents at each event to enable centralized-style replanning.
To establish the global connectivity,
agents synchronize at the designated \textit{communication events}
defined as $\mathbf{C} \triangleq \left\{ \mathbf{c}_i \mid \forall i \in \mathcal{N} \right\}$,
where each agent-specific event $\mathbf{c}_i \triangleq (t^{\texttt{c}}_i, p^{\texttt{c}}_i)$
specifies the communication time $t^{\texttt{c}}_i$ and position $p^{\texttt{c}}_i$.
The global event $\mathbf{C}$ occurs
when all agents simultaneously reach their respective communication positions at their common times $t^{\texttt{c}}$,
i.e., $t^{\texttt{c}} = t^{\texttt{c}}_i, \, \forall i \in \mathcal{N}$.
A valid communication event $\mathbf{C}$ must ensure that:
\begin{equation}\label{eq:comm-graph}
  G(t^{\texttt{c}}) \in \mathcal{G}_{\texttt{c}},
\end{equation}
where $t^{\texttt{c}}$ is the communication time.
Once $G(t^{\texttt{c}})$ is connected, agents exchange critical information including:
their individually detected tasks $\Omega_t^i$,
locally generated plans $\tau_i$,
and future communication events~$\mathbf{c}_i$.
Subsequently, the local plan for agent $i$ extends to
$
  \tau_i \triangleq (t^1_i,\,g^1_i,\,a^1_i)(t^2_i,\,g^2_i,\,a^2_i)\cdots\mathbf{c}_i
$.

\subsection{Problem Statement}\label{subsec:prob}

The overall objective is to synthesize the collective
plans $J \triangleq \{\tau_i\}_{i \in \mathcal{N}}$ so as to maximize
the long-run average rate of task completion over an infinite
horizon.
For a given collective plan $J$ and time $T>0$, define
  $N_T(J) \triangleq
  \left|
    \left\{
      \omega_m \in \boldsymbol{\Omega}_T
      \,\middle|\,
      t_m^{\texttt{f}}(J) \le T
    \right\}
  \right|$,
which denotes the number of tasks completed by time $T$ under plan $J$.
The long-run average task completion rate of $J$ is then
$
  \eta(J) \triangleq
  \mathbf{lim}_{T \to \infty}
  \frac{N_T(J)}{T}$.
  Thus, the overall problem is as follows:
\begin{equation}\label{eq:objective}
  \begin{aligned}
    & \underset{J}{\mathbf{max}} \quad
      \eta(J) \\
    & \mathbf{s.t.} \quad
      \eqref{eq:task-constraints} - \eqref{eq:comm-graph}, \;\forall i \in \mathcal{N},
  \end{aligned}
\end{equation}
where the collective plan $J$ that maximizes
the average rate~$\eta(J)$;
$\eqref{eq:task-constraints}$ ensures that all
temporal relations between tasks are satisfied;
and $\eqref{eq:comm-graph}$ guarantees that the team achieves
intermittent global connectivity at communication events.

\vspace{-0.8em}
\section{Proposed Solution}\label{sec:solution}

A unified framework co-optimizes collaborative task planning and communication in dynamic, unknown environments.
It combines a Branch-and-Bound method for temporally constrained planning and communication scheduling
in Sec. \ref{subsec:coord-framework}, an iterative approach for optimizing communication events
in Sec. \ref{subsec:comm-event}, and an online execution scheme with complexity analysis
and theoretical guarantees in Sec. \ref{subsec:online}.

\vspace{-0.8em}
\subsection{Adaptive Coordination Framework}
\label{subsec:coord-framework}

\subsubsection{Design of Adaptive Optimization Objective}
\label{subsubsec:new-obj}

To handle unknown spatiotemporal task distributions while still achieving the global
objective~\eqref{eq:objective}, an adaptive optimization objective is formulated to
maximize task completion efficiency within each execution horizon. The coordination
framework is triggered at every communication event occurring at time~$t$. Based on the
currently detected task set $\mathbf{\Omega}_t$, the framework jointly optimizes the
future local plans~$J$ together with the scheduling of the next communication
event~$\mathbf{c}$. The team-wise intermittent communication protocol
guarantees that all tasks assigned in $\{\tau_i\}$ during the interval $[t, t^{\texttt{c}}]$
are completed before the next re-planning phase at $t^{\texttt{c}}$.
The adaptive objective
is expressed as the maximization of the task completion rate per unit time in the interval
$[t, t^{\texttt{c}}]$, as an \emph{local approximation} of $\eta(J)$ in~\eqref{eq:objective}:
\begin{equation}\label{eq:adaptive_objective}
\begin{aligned}
& \underset{{\{\tau_i\}, t^{\texttt{c}}}}{\textbf{max}} \quad
\frac{ \left| \left\{ \omega_m \in \mathbf{\Omega}_t \mid t^{\texttt{f}}_m \leq t^{\texttt{c}} \right\} \right| }
{t^{\texttt{c}} - t} \\
\textbf{s.t.} \quad  &
G(t^{\texttt{c}}) \in \mathcal{G}_{\texttt{c}}, \, \mathcal{C}_t \text{ holds};
t^{\texttt{c}} \geq t^{\texttt{f}}_{i, \ell} > t,\\
&t^{\texttt{s}}_{i, \ell+1} \geq t^{\texttt{f}}_{i, \ell} + T_{\text{travel}}(\omega^\ell_i, \omega^{\ell+1}_i), \omega^\ell_i \in \mathbf{\Omega}_t \,
\end{aligned}
\end{equation}
where $t^{\texttt{f}}_m$ is the completion time of task $\omega_m$, $t^{\texttt{c}}$ denotes
the scheduled time of the next communication event, and $G(t^{\texttt{c}}) \in \mathcal{G}_{\texttt{c}}$
ensures that global communication occurs at $t^{\texttt{c}}$.


\subsubsection{Algorithm Description}
\label{subsubsec:algorithm-description}

\begin{algorithm}[!t]
    \caption{$\texttt{CoCoPlan}(\cdot)$}
    \label{alg:cocoplan}
    \SetAlgoLined
    \SetAlgoNlRelativeSize{-1} 
    \DontPrintSemicolon
    \KwIn{Robot team $\mathcal{N}$, Detected tasks $\mathbf{\Omega}_t$, Time budget $T'$} \label{line:input}
    \KwOut{Best collective plan  $J^\star$}   \label{line:output}

    Initialize $\mathcal{T} \gets \emptyset$, $\mathcal{Q}$ (max-heap by UB) \; \label{line:init}
    Create $\nu_0$: $J_{\nu_0} \gets \emptyset$ \; \label{line:create_nu0}
    $(LB_{\nu_0}, J_{\nu_0}) \gets \texttt{LowBound}(\nu_0, \mathbf{\Omega}_t, \mathcal{N})$ \; \label{line:lower_nu0}
    $UB_{\nu_0} \gets \texttt{UpBound}(\nu_0, \mathbf{\Omega}_t, \mathcal{N})$ \; \label{line:upper_nu0}
    $\mathcal{T} \gets \mathcal{T} \cup \{\nu_0\}$, $\mathcal{Q}$.insert($\nu_0$, $UB_{\nu_0}$) \; \label{line:insert_nu0}
    $LB^\star \gets LB_{\nu_0}$, $J^\star \gets J_{\nu_0}$, \label{line:init_best}
    $t_{\texttt{start}} \gets \text{current time}$ \; \label{init-time}

    \While{$\mathcal{Q} \neq \emptyset$ \textbf{and} $(t- t_{\texttt{start}} < T')$}{ \label{line:while_start}
        $\nu \leftarrow \mathcal{Q}$.extractMax()  \label{line:extract}

        \If{$UB_{\nu} > LB^\star$}{ \label{line:if_ub}
            \If{$LB_{\nu} > LB^\star $}{$LB^\star \gets LB_{\nu}$, $J^\star \gets J_{\nu}$} \label{line:update_best}

            $\mathcal{F} \leftarrow \texttt{GetTasks}(\nu, \mathbf{\Omega}_t)$ \; \label{line:feasible}
            \ForEach{$\omega \in \mathcal{F}$}{ \label{line:for_start}
                $\nu_+ \leftarrow \texttt{ExpNode}(\nu, \omega)$ \; \label{line:expand}
                $(LB_{\nu_+}, J_{\nu_+}) \gets \texttt{LowBound}(\nu_+, \mathbf{\Omega}_t, \mathcal{N})$ \; \label{line:lower_new}
                $UB_{\nu_+} \gets \texttt{UpBound}(\nu_+, \mathbf{\Omega}_t, \mathcal{N})$ \; \label{line:upper_new}
                $\mathcal{T} \gets \mathcal{T} \cup \{\nu_+\}$ \; \label{line:insert_new}
                \If{$LB_{\nu_+} > LB^\star$}{$LB^\star \gets LB_{\nu_+}$, $J^\star \gets J_{\nu_+}$} \label{line:update_new}
                \If{$UB_{\nu_+} > LB^\star$}{$\mathcal{Q}$.insert($\nu_+, UB_{\nu_+}$)} \label{line:insert_new}
            } \label{line:for_end}
        } \label{line:if_end}
    } \label{line:while_end}
    \Return $J^\star$ \label{line:return}
\end{algorithm}
\setlength{\textfloatsep}{5pt}

The proposed $\texttt{CoCoPlan}$ in Alg.~\ref{alg:cocoplan}
uses a branch-and-bound (BnB) search to jointly optimize task assignments
and communication events for the entire team.
Each node $\nu$ in the BnB tree corresponds to a collective plan
$J_\nu \triangleq [(\omega_{i_1}^1, \omega_{i_1}^2, \cdots, \mathbf{c}_{i_1}),
(\omega_{i_2}^1, \cdots, \mathbf{c}_{i_2}), \cdots]$,
i.e., agent-specific task sequences interleaved with communication events
$\mathbf{c}$. For brevity, the local plan elements $(t^\ell_i, g^\ell_i, a^\ell_i)$
are associated with the task $\omega^\ell_i$ in $J_\nu$.
For each node $\nu$, we maintain a lower bound $LB_\nu$ and an upper bound $UB_\nu$
on the objective value; 
nodes are pruned if $UB_\nu < LB^\star$, 
where $LB^\star$ is the current best lower bound.
The objective of each node is given by~\eqref{eq:adaptive_objective},
representing the task completion rate up to the next communication event.
The search continues until all nodes are expanded or pruned, or a time limit
is reached. It then returns the best collective plan $J^\star$ found within the
computation budget.
The algorithm follows five key stages:

\textbf{(I) Initialization} (Lines~\ref{line:init}--\ref{line:init_best}):
The root node $\nu_0$ is created with an empty plan, and its bounds
$LB_{\nu_0}$ and $UB_{\nu_0}$ are computed by
$\texttt{LowBound}$ and $\texttt{UpBound}$, providing pessimistic and
optimistic estimates of achievable performance. These values initialize
the max-heap $\mathcal{Q}$, which stores nodes ordered by $UB_\nu$, and
the incumbent solution $(LB^\star, J^\star)$. A timer $t_{\texttt{start}}$
is started to enforce the computation budget $T'$.

\textbf{(II) Selection} (Lines~\ref{line:while_start}--\ref{line:if_ub}):
At each iteration, the node $\nu$ with the largest $UB_\nu$ is extracted
from $\mathcal{Q}$ (Line~\ref{line:extract}), following a best-first
strategy that prioritizes nodes with the highest potential to improve
the incumbent. If $UB_{\nu} \leq LB^\star$, then $\nu$ and its subtree
are pruned, since they cannot yield a better solution. This pruning
significantly reduces the explored search space.

\textbf{(III) Expansion} (Lines~\ref{line:feasible}--\ref{line:for_end}):
For each surviving node, feasible tasks $\mathcal{F}$ that respect the
temporal constraints $\mathcal{C}_t$ are identified
(Line~\ref{line:feasible}).
The $\texttt{ExpNode}$ function (Line~\ref{line:expand})
generates child nodes $\nu_+$ by assigning each task
$\omega \in \mathcal{F}$ to eligible agents and inserting it in their
plans before the next communication event~$\mathbf{c}_{i_k}$.
This couples task allocation with communication scheduling. Each child
is then evaluated to obtain new bounds $LB_{\nu_+}$ and $UB_{\nu_+}$.

\textbf{(IV) Bounding} (Lines~\ref{line:lower_nu0}--\ref{line:upper_new}):
Each child $\nu_+$ undergoes iterative bound refinement starting from
its plan $J_{\nu_+}$.
As illustrated in Fig.~\ref{fig:lower-bound}, at iteration $k$
a feasible task $\omega^{(k)}$ is chosen from
$\mathbf{\Omega}_t \setminus \mathbf{\Omega}_{\nu_+}$ and assigned to
an eligible agent group $\mathcal{A} \in \mathcal{A}_k$ that minimizes
the maximum travel time to the task location:
\begin{equation}\label{eq:task-assign}
\underset{\mathcal{A} \in \mathcal{A}_k}{\mathbf{min}}
\left\{
  \underset{i \in \mathcal{A}}{\mathbf{max}}\,
  T_{\text{travel}}(p^{\texttt{f}}_i, \omega^{(k)})
\right\},
\end{equation}
yielding an updated plan $J^{(k)}$ for node $\nu^{(k)}$.
Its completion time is computed in two phases:
(1) \textit{Task phase}: non-communication tasks are scheduled by solving:
$
\mathbf{min} \;
{\mathbf{max}_{\omega_m \in\, \mathbf{\Omega}_{\nu^{(k)}}}} t_m^{\texttt{f}},
$
using a standard LP solver such as OR-Tools~\cite{glop}, temporarily
ignoring communication; and (2) \textit{communication phase}: given the
resulting task schedule, $\texttt{ComOpt}$ optimizes communication
events and returns $t^{\texttt{c}}_{(k)}$.
The corresponding candidate
rate is
$b^{(k)} = |\omega_m \in \mathbf{\Omega}_{\nu^{(k)}}| / t^{\texttt{c}}_{(k)}$.
Tasks are added while $b^{(k)}$ improves, and the lower bound
$LB_{\nu_+}$ is set to $\max_j b^{(j)}$. The upper bound $UB_{\nu_+}$
is obtained by applying the same procedure but assuming instantaneous
agent movement, giving an optimistic estimate for pruning.

\textbf{(V) Termination} (Line~\ref{line:return}):
The algorithm stops when the heap $\mathcal{Q}$ is empty or the time
budget $T'$ is reached, and returns the incumbent plan $J^\star$ as the
best solution found. Since bounds and solutions are refined throughout
the search, the method is anytime: even early termination yields a
feasible plan.

\begin{figure}[!t]
    \centering
    \includegraphics[width=0.85\linewidth]{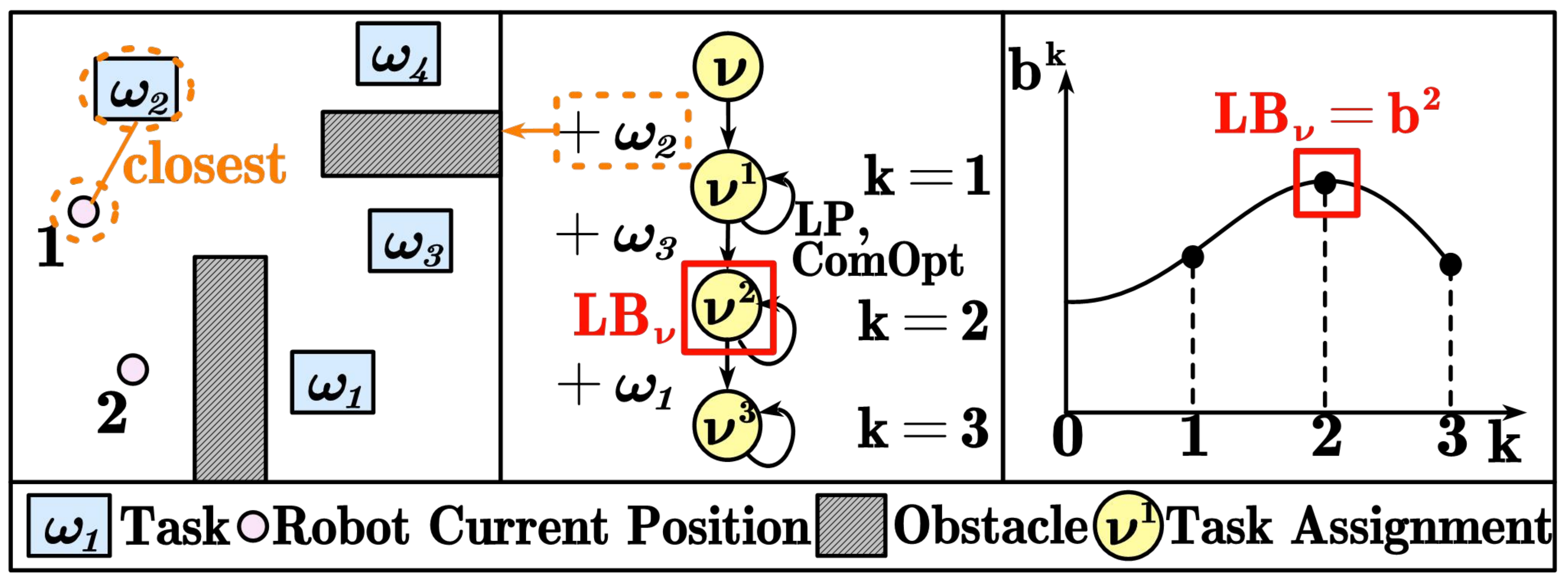}
    \vspace{-2mm}
    \caption{
      Illustration of the iterative process for computing lower bounds,
      after inserting new tasks in the collective plan.
    }
    \label{fig:lower-bound}
    \vspace{0mm}
  \end{figure}

\vspace{-0.8em} 
\subsection{Optimization of Communication Events}
\label{subsec:comm-event}

\subsubsection{Problem of Communication Events Optimization}
\label{subsubsec:comm-event-opt}

Given the non-communication solution $J_{\nu^+}$ obtained in the previous step,
let $t^{\texttt{f}}_i$ denote the completion time of the last non-communication task
in the plan of agent $i$ as $\tau_i$, and let $p^{\texttt{f}}_i \in \mathcal{W}$ be its position
at time $t^{\texttt{f}}_i$. The subsequent communication event must occur at a time
$t^{\texttt{c}} > \max_{i \in \mathcal{N}} t^{\texttt{f}}_i$,
with agents located at positions $\{p^{\texttt{c}}_i\}$, where $p^{\texttt{c}}_i \in \mathcal{W}_i$,
and the induced communication graph $G(t^{\texttt{c}})$ satisfies the connectivity requirement
$G(t^{\texttt{c}}) \in \mathcal{G}_{\texttt{c}}$.
The optimization problem aims to minimize the maximum transition time required for agents
to travel from their last task locations to the communication point, ensuring both feasibility
and efficiency:
\begin{equation}\label{eq:comm-opt}
\begin{aligned}
& \underset{t^{\texttt{c}},\, {\{}p^{\texttt{c}}_i\}}{\textbf{min}}\,\big{\{}
\underset{{i \in \mathcal{N}}}{\textbf{max}}  \left( t^{\texttt{c}} - t^{\texttt{f}}_i \right)\big{\}} \\
\textbf{s.t.} \quad
& t^{\texttt{c}} \geq t^{\texttt{f}}_i + T_{\text{travel}}(p^{\texttt{f}}_i, p^{\texttt{c}}_i),
\quad \forall i \in \mathcal{N}, \\
& G(t^{\texttt{c}}) \in \mathcal{G}_{\texttt{c}},
\end{aligned}
\end{equation}
where the first constraint guarantees sufficient time for each agent to navigate from
its last task location $p^{\texttt{f}}_i$ to its communication position
$p^{\texttt{c}}_i$;
and the second enforces that the team remains
connected at the communication time $t^{\texttt{c}}$.
This formulation explicitly balances task feasibility with connectivity, ensuring that
communication is synchronized across the team while minimizing idle time.

\subsubsection{Algorithm Description}
\label{subsubsec:comm-event-algorithm}

\begin{algorithm}[!t]
    \caption{$\texttt{ComOpt}(\cdot)$}
        \label{alg:comopt}
        \SetAlgoLined
        \DontPrintSemicolon
        \setcounter{AlgoLine}{0} 
        \KwIn{Robot team $\mathcal{N}$, set of last tasks~$\{t^{\texttt{f}}_i\}$, time budget $T'$, threshold $\delta$}
        \KwOut{Communication event $\mathbf{C}_{t}$}

        $i_\ell \gets \textbf{argmax}_i t^{\texttt{f}}_i$ \; \label{line:1}
        $p_0 \gets p^{\texttt{f}}_{i_\ell}$, $t^{\texttt{c}} \gets t^{\texttt{f}}_{i_\ell}$, $t^{\texttt{b}} \gets t^{\texttt{c}}$ \; \label{line:2}
        $\mathbf{C}_{t} \gets \left\{ (t^{\texttt{b}}, p_0) \mid \forall i \in \mathcal{N}, \forall j>i \right\}$ \; \label{line:3}
        $t_{\texttt{start}} \gets \text{current time}$ \; \label{line:4}
        \While{$(t - t_{\texttt{start}} < T')$ \textbf{and} $(|t^{\texttt{c}} - t^{\texttt{b}}| \geq \Delta)$}{ \label{line:5}
            $i_{\texttt{e}} \gets \textbf{argmin}_i (t^{\texttt{f}}_i + \texttt{A}^\star(p^{\texttt{f}}_i, p_0))$ \; \label{line:6}
            $j_n \gets \textbf{argmin}_{j \neq i_{\texttt{e}}} \texttt{A}^\star(p^{\texttt{f}}_{i_{\texttt{e}}}, p_j^c)$ \; \label{line:7}
            $p^{\texttt{c}} \gets \texttt{SelCom}(p^{\texttt{f}}_{i_{\texttt{e}}}, p_{j_n}^c)$ \; \label{line:8}
            $t_{\texttt{a}} \gets t^{\texttt{f}}_{i_{\texttt{e}}} + \texttt{A}^\star(p^{\texttt{f}}_{i_{\texttt{e}}}, p^{\texttt{c}})$ \; \label{line:9}
            $t^{\texttt{c}}_+ \gets \textbf{max}( t_{\texttt{a}}, \textbf{max}_{k \neq i_{\texttt{e}}} (t^{\texttt{f}}_k + \texttt{A}^\star(p^{\texttt{f}}_k, p_k^c)) )$ \; \label{line:10}
            \If{$(Q_{c, p^c_{j_n}} > \delta)$ \textbf{and} $(t^{\texttt{c}}_+ < t^{\texttt{b}})$}{ \label{line:11}
                $p_{i_{\texttt{e}}}^c \gets p^{\texttt{c}}$ \; \label{line:12}
                \eIf{$|t^{\texttt{b}} - t^{\texttt{c}}_+| < \delta$}{ \label{line:13}
                    $t^{\texttt{c}} \gets t^{\texttt{c}}_+$, $t^{\texttt{b}} \gets t^{\texttt{c}}_+$; \textbf{break} \; \label{line:14}
                }{
                    $t^{\texttt{c}} \gets t^{\texttt{c}}_+$, $t^{\texttt{b}} \gets t^{\texttt{c}}_+$ \; \label{line:16}
                }
            }
        }
        \Return $\mathbf{C}_{t} = \{ (t^{\texttt{b}}, p_i^c) \mid \forall i\}$ \label{line:17}
\end{algorithm}

The $\texttt{ComOpt}(\cdot)$ algorithm in Alg.~\ref{alg:comopt} addresses the optimization
problem in~\eqref{eq:comm-opt}. The procedure begins by identifying the agent $i_\ell$ that
finishes its last task the latest in Line~\ref{line:1}, since all agents must eventually
wait for this agent before communication. The candidate communication location~$p_0$ is
set to~$p^{\texttt{f}}_{i_\ell}$, and the initial communication time is initialized to
$t^{\texttt{b}} = t^{\texttt{c}} = t^{\texttt{f}}_{i_\ell}$ in Line~\ref{line:2}.
Accordingly, the event $\mathbf{C}_{t}$ is initialized that all
agents converge to~$p_0$ at $t^{\texttt{b}}$.

Within the computational budget $T'$, the algorithm iteratively refines the initial guess
until the gap $|t^{\texttt{c}} - t^{\texttt{b}}|$ is below the convergence threshold $\Delta$
as stated in Line~\ref{line:5}. As illustrated in Fig.~\ref{fig:comm-event}, each iteration
proceeds as follows. First, the earliest-arriving agent $i_{\texttt{e}}$ is selected in
Line~\ref{line:6} by minimizing $t^{\texttt{f}}_i + \texttt{A}^\star(p^{\texttt{f}}_i, p_0)$.
Among the agents that finish last and the agents whose communication has already optimized, we identify the nearest neighbor $j_n$ in Line~\ref{line:7} using the shortest
$\texttt{A}^\star$ path from $p^{\texttt{f}}_{i_{\texttt{e}}}$ to $p_{j_n}^c$. A candidate position $p^{\texttt{c}}$ is sampled
between $p^{\texttt{f}}_{i_{\texttt{e}}}$ and $p_{j_n}^c$ in Line~\ref{line:8} using the
$\texttt{SelCom}$ function, and the corresponding arrival time $t_{\texttt{a}}$ is computed
in Line~\ref{line:9}. The potential communication time $t^{\texttt{c}}_+$ is then defined in
Line~\ref{line:10} as the maximum of $t_{\texttt{a}}$ and the arrival times of all other
agents at their current positions.
If the the communication quality between candidate $p^{\texttt{c}}$ and $p_{j_n}^c$ exceeds
the threshold $\delta$, and the updated communication time $t^{\texttt{c}}_+$ is earlier
than the current best $t^{\texttt{b}}$, then the location~$p^{\texttt{c}}$ is accepted as the new
position of agent $i_{\texttt{e}}$ in Line~\ref{line:12}. The communication time is updated
accordingly in Line~\ref{line:14} and Line~\ref{line:16}, and convergence is reached once the difference
$|t^{\texttt{b}} - t^{\texttt{c}}_+|$ falls below~$\delta$, which terminates the loop early.
Lastly, the algorithm outputs the optimized communication event
$\mathbf{C}_{t}$, including the
agent positions $\{p_i^{\texttt{c}}\}$ and the synchronized time $t^{\texttt{b}}$.
This event satisfies $G(t^{\texttt{b}}) \in \mathcal{G}_{\texttt{c}}$
and minimizes the maximum delay $\textbf{max}_i (t^{\texttt{b}} - t^{\texttt{f}}_i)$,
as required by~\eqref{eq:comm-opt}.

\vspace{-0.8em} 
\subsection{Overall Analysis}
\label{subsec:online}

\begin{figure}[!t]
    \centering
    \includegraphics[width=0.85 \linewidth, height=0.42\linewidth]{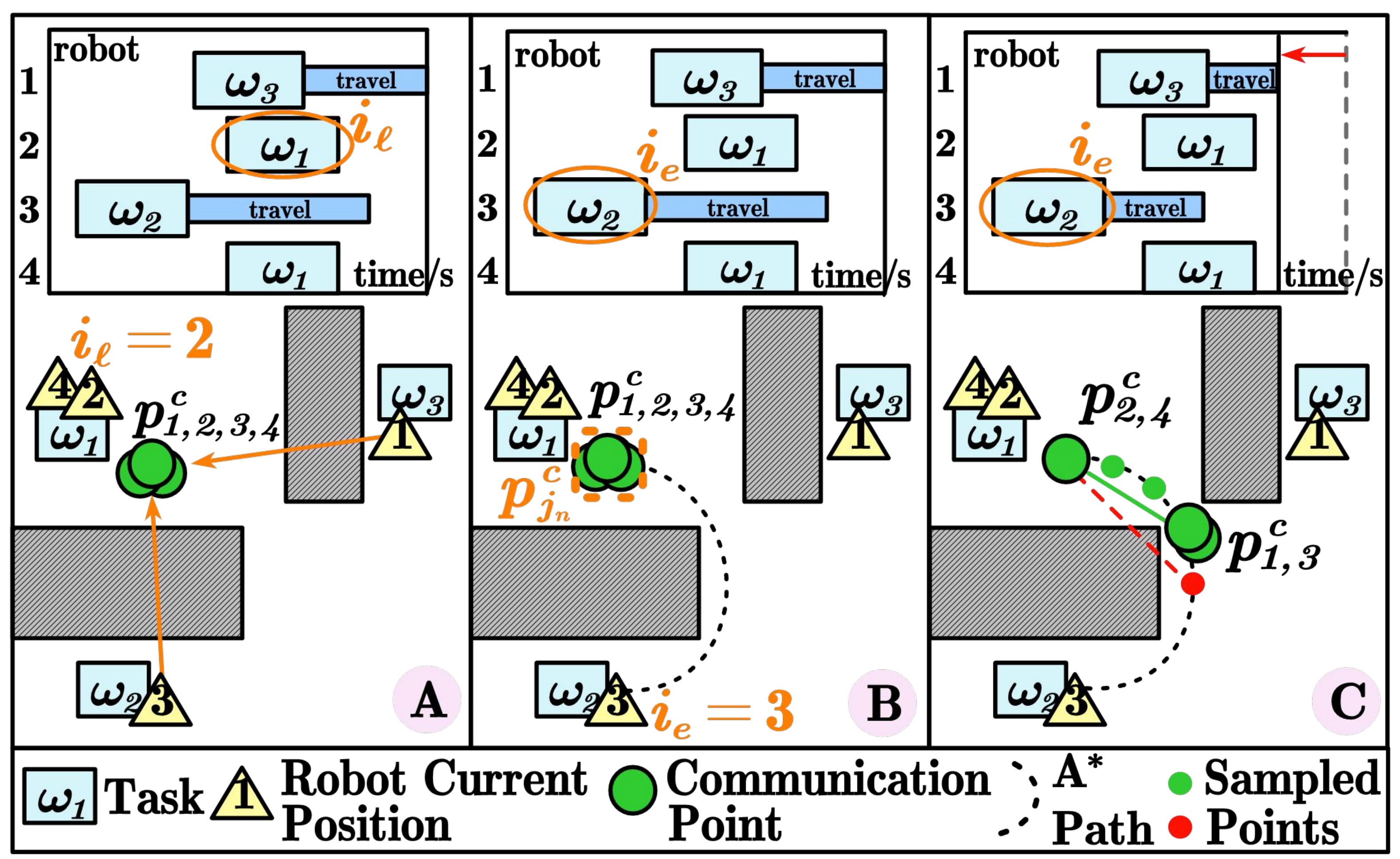}
    \vspace{-3mm}
    \caption{Optimization of communication events, showing the selection of the earliest-arriving
agent and the sampling of candidate communication points to achieve synchronized connectivity.}
    \label{fig:comm-event}
    \vspace{0mm}
  \end{figure}
\subsubsection{Online Execution}
\label{subsubsec:online-execution}

To handle dynamically emerging tasks, we adopt an online execution
scheme. As illustrated in Fig.~\ref{fig:framework}, agents discover new
tasks while moving, incrementally updating the task set
$\mathbf{\Omega}_t$. Each agent executes its current local plan
$\tau_i$ and progresses toward its next communication location.

To handle uncertainties such as varying
execution times and navigation delays, we adopt simple
adaptation mechanisms: for each assigned task, the corresponding agents
wait at the task region until all members have arrived before
execution; at each communication event, all agents wait until every
agent has reached its designated communication location, ensuring
$G(t^{\texttt{c}}) \in \mathcal{G}_{\texttt{c}}$. 
This maintains both temporal and communication constraints, 
even when timing deviations occur.
At each communication event, agents exchange information about newly
discovered tasks. A designated agent aggregates these into
$\mathbf{\Omega}_t$ and invokes $\texttt{CoCoPlan}$ to compute updated
collective plans and the next communication event~$\mathbf{C}$. The
team then follows the new plan, and this cycle of execution,
communication, and replanning repeats until the mission ends.

\subsubsection{Complexity Analysis}
\label{subsubsec:complexity}

The complexity of Alg.~\ref{alg:cocoplan} is dominated by its
branch-and-bound (BnB) search. With $m$ tasks and $N$ agents, a single
expansion step can generate up to $\mathcal{O}(m \cdot N!)$ child nodes,
leading to a worst-case total of $\mathcal{O}((m \cdot N!)^m)$ nodes.
Per node, task expansion and bound evaluation are dominated by
$\mathcal{O}(m \cdot N!)$, so the overall worst-case complexity remains
$\mathcal{O}((m \cdot N!)^m)$. In practice, pruning substantially
reduces the number of explored nodes.


\begin{lemma}\label{lem:alg2_graph_connect}
Alg.~\ref{alg:comopt} returns communication locations such that the
communication graph $G(t^{\texttt{c}})$ is connected.
\end{lemma}
\begin{proof}
(Sketch)
All agents are initialized at the same communication point $p_0$, so
$G(t^{\texttt{c}})$ is connected. At each iteration, the algorithm
updates a single agent $i_e$ to a new point $p_{i_e}$ that preserves a
communication link to already–fixed point~$p_{j_n}$, thus adding a
node adjacent to the existing connected set. Repeating this for all
agents preserves connectivity by induction.
\end{proof}

\begin{theorem}
\label{thm:cocoplan_opt}
Assume that communication events recur indefinitely and that at
every communication time Alg.~\ref{alg:cocoplan} is solved to
optimality for the adaptive objective~\eqref{eq:adaptive_objective}.
Then on each communication cycle $[t_k^{\texttt{c}}, t_{k+1}^{\texttt{c}})$
the returned  plan is feasible w.r.t.
\eqref{eq:task-constraints} and \eqref{eq:comm-graph},
and locally maximizes the task completion rate over that cycle.
Moreover,
\eqref{eq:adaptive_objective} provides a local approximation of the
global infinite-horizon objective~\eqref{eq:objective}; 
Under stationary task arrival and ergodicity assumptions, 
the two optima coincide.
\end{theorem}
\begin{proof}
(Sketch)
Feasibility on each cycle follows because Alg.~\ref{alg:cocoplan}
enforces the temporal constraints~\eqref{eq:task-constraints} and calls
Alg.~\ref{alg:comopt}, which guarantees connectivity via
Lemma~\ref{lem:alg2_graph_connect}. Solving the branch-and-bound search
to optimality ensures that the returned plan maximizes
\eqref{eq:adaptive_objective} over that cycle, giving a locally optimal
approximation of~\eqref{eq:objective}. Under stationary, constant-rate
task discovery and ergodic cycles, the renewal–reward theorem implies
that the long-run average completion rate equals the time average of the
per-cycle rates, so local and global optimality are equivalent.
\end{proof}

\vspace{-0.8em} 
\section{Numerical Experiments} \label{sec:experiments}

To validate the proposed method, extensive numerical experiments
are conducted across multiple scenarios. The method is implemented in \texttt{Python3}
within the \texttt{ROS} framework, and evaluated on a workstation equipped with an
Intel(R) i7 24-core CPU and an RTX-4070Ti GPU.
Simulation and hardware videos are
provided in the supplementary material.

\vspace{-0.8em} 
\subsection{Workspace and Task Description}
\label{subsec:workspace}

Experiments employ $N=10$ heterogeneous robots in two challenging environments.
(I) The {DARPA SubT challenge environment}, measuring $40~\text{m} \times 60~\text{m}$,
contains narrow passages that test maneuverability under constrained conditions.
(II) The {subterranean caves}, with a size of $40~\text{m} \times 70~\text{m}$,
include irregular obstacles that serve as a complex setting for validation.
All robots operate with identical parameters: velocity range with maximum $v^{\texttt{max}}_i = 2$~m/s, sensing range $d_i = 8$~m, and communication quality determined by the model in Section~\ref{subsec:comm-const}.
The method activates during communication events.
Task sequences $\tau_i$ are executed using the native SLAM-based navigation package.

\vspace{-0.8em} 
\subsection{Simulation Results}
\label{subsec:simulation-results}

\begin{figure}[!t]
    \centering
    \includegraphics[width=0.98\linewidth]{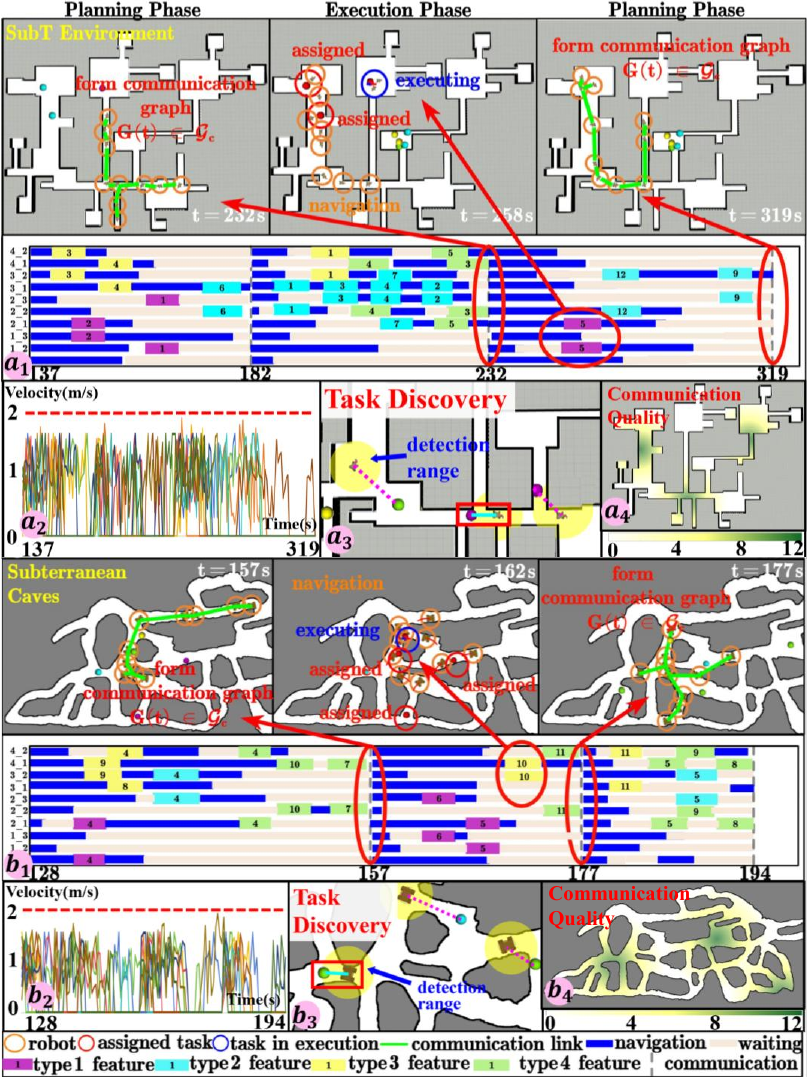}
    \vspace{-3mm}
    \caption{Results of the proposed method in simulation, including representative
    keyframes and task planning outcomes in the {DARPA SubT challenge environment} (\textbf{top})
    and the {subterranean caves} (\textbf{bottom})}.
    \label{fig:sim-result}
\end{figure}

The proposed method is evaluated in both environments, as shown in Fig.~\ref{fig:sim-result}.
Each algorithm invocation takes~15s of computation with~100 nodes. In the SubT scenario,
four task types are considered: Type~1 fault detection precedes Type~2 cleaning, and Type~3
indoor delivery precedes Type~4 equipment maintenance. In Fig.~\ref{fig:sim-result}(a), at 232s,
the robots establish communication and synchronize their plans. By 258s, execution begins, with
2 robots performing fault detection. At 319s, the communication topology is reorganized for
the next planning cycle. In the subterranean caves, the task set includes: Type~1 rescue precedes
Type~2 cargo transportation, and Type~3 exploration precedes Type~4 mineral sampling. In
Fig.~\ref{fig:sim-result}(b), a communication graph forms at 157s to initiate task allocation. By 162s,
robots navigate to their goals, with exploration tasks completed before mineral sampling due
to temporal constraints. At 177s, the communication topology is reorganized for continued coordination.
These results demonstrate real-time task allocation and communication topology adaptation.
The system achieves 4 communication events and 17 tasks in SubT, and 4 communication events
with 16 tasks in subterranean caves.

\vspace{-0.8em} 
\subsection{Comparison with Baselines}
\label{subsec:results}

To evaluate the proposed method, it is compared against six baselines.
(I) FIX triggers planning when accumulated tasks reach a threshold $N$, and is
adapted to enforce temporal constraints.
(II) FPMR~\cite{zijlstra2024multi} uses a fixed communication point; for fairness,
it is applied only to communication optimization while retaining the proposed
task allocation strategy.
(III) FRDT~\cite{kantaros2019temporal} gathers robots around a fixed leader to
communicate, and is likewise used only for communication optimization.
(IV) FIMR~\cite{aragues2020intermittent} plans at regular intervals $T_{\texttt{c}}$,
used only for communication optimization to preserve adaptivity.
(V) RING~\cite{guo2018multirobot} pairs robots with fixed partners and propagates
information along the chain; it is modified to handle temporal and capability
constraints.
(VI) Greedy communicates and executes any feasible tasks whenever robots meet,
and is adapted to account for capability constraints during task allocation.
Four metrics are reported: finished tasks, communication count,
communication interval, and $t_{\texttt{c}}-t_f$, the gap between the
last task completion and the next communication. All methods are tested in both
scenarios over three trials, each combining spatial patterns (clustered, uniform,
sparse) and temporal patterns (spiky, uniform, low-frequency). Since the selected
baselines were developed under slightly different (often simpler) settings, the
comparison is intended to highlight performance in complex unknown scenarios
rather than claim narrow superiority.

\subsubsection{Performance in Nominal Scenarios}
\label{subsubsec:comparison-existing-methods}

\begin{table}[!t]
  \centering
  \caption{\uppercase{Comparison of baselines across two scenarios.}}
  \footnotesize
  \setlength{\tabcolsep}{3pt}
  \renewcommand{\arraystretch}{1.1}
  \begin{tabular}{@{}l l c c c c@{}}
  \toprule
  \toprule
  \makecell{\textbf{Env.}} & \makecell{\textbf{Strategy}} & \makecell{\textbf{Finished}\\\textbf{Tasks (\#)}} & \makecell{\textbf{Comm.}\\\textbf{Num. (\#)}} & \makecell{\textbf{Comm.}\\\textbf{Int. (s)}} & \makecell{\textbf{$t_{\texttt{c}}-t_f$}\\\textbf{(s)}} \\
  \midrule
  \multirow{9}{*}{SubT}
    & FIX (N=3)                   & 74.7±3.1   & 34.0±0.0   & 22.2±10.4   & 2.6±4.8 \\
    & FIX (N=10)                  & 95.3±6.6   & 12.0±0.0   & 51.8±18.2   & 3.0±4.6 \\
    & FPMR                         & 89.7±1.0   & 9.3±1.3    & 78.7±60.3   & 19.3±8.3 \\
    & FRDT                         & 94.3±3.1   & 10.3±1.9   & 69.8±54.8   & 17.2±12.9 \\
    & FIMR ($T_{\texttt{c}}=35$)   & 96.0±5.0   & 17.7±1.3   & 35.0±0.0    & 6.2±6.1 \\
    & FIMR ($T_{\texttt{c}}=80$)   & 87.7±4.0   & 9.0±0.0    & 80.0±0.0    & 16.4±18.9 \\
    & RING                        & 52.3±6.8   & 7.0±0.8    & 156.1±75.1  & 85.0±42.7 \\
    & Greedy                      & 30.0±10.6   & 166.3±10.8  & --  & -- \\
    & \textbf{Ours}                & \textbf{99.0±0.8} & \textbf{13.0±0.8} & \textbf{46.7±44.0} & \textbf{2.8±2.5} \\
    \midrule
  \multirow{9}{*}{Caves}
    & FIX (N=3)                   & 81.7±5.9   & 34.7±0.4   & 21.6±11.6   & 2.6±5.2 \\
    & FIX (N=10)                  & 95.7±4.8   & 12.0±0.0   & 52.8±17.8   & 4.2±6.1 \\
    & FPMR                         & 96.7±2.5   & 10.3±1.7   & 68.3±55.6   & 19.6±9.9 \\
    & FRDT                         & 84.3±1.0   & 10.0±0.8   & 82.0±76.6   & 18.2±13.1 \\
    & FIMR ($T_{\texttt{c}}=35$)   & 83.7±4.5   & 20.7±0.4   & 35.0±0.0    & 5.9±5.9 \\
    & FIMR ($T_{\texttt{c}}=80$)   & 90.7±6.8   & 8.7±0.4    & 80.0±0.0    & 19.8±22.6 \\
    & RING                        & 49.0±8.3   & 7.0±0.8    & 183.0±84.8  & 106.2±53.6 \\
    & Greedy                      & 57.7±3.3   & 166.7±17.5  & --  & -- \\
    & \textbf{Ours}                & \textbf{98.7±0.4} & \textbf{13.7±0.4} & \textbf{45.1±43.4} & \textbf{2.2±4.6} \\
  \bottomrule
  \bottomrule
  \end{tabular}
  \label{tab:comm_results}
\end{table}

As shown in Table~\ref{tab:comm_results}, the proposed method achieves nearly
double the task completion of RING, reaching 99.0 vs. 52.3 in SubT and 98.7
vs. 49.0 in the caves. It also completes nearly three times and twice as many
tasks as Greedy, while requiring 92.2\% and 91.8\% fewer communication events
in the two scenarios. Compared to FIX with $N=3$, substantially more tasks are
completed (99.0 vs. 74.7 and 98.7 vs. 81.7). Against FIX with $N=10$, the
adaptive objective markedly reduces variance (0.8 vs. 6.6 and 0.4 vs. 4.8).
FIMR with $T_{\texttt{c}}=35$s exhibits instability across environments, with a
14\% drop in completion, whereas the proposed method remains consistently high.
FIMR with $T_{\texttt{c}}=80$s incurs long idle times and large variance.
FPMR and FRDT are constrained by fixed communication, causing significant delays
and exposing the drawback of rigid synchronization relative to adaptive
coordination. Overall, the proposed method consistently delivers high task
throughput while limiting unnecessary coordination overhead across diverse
task distributions.

\begin{figure}[th!]
    \centering
    \includegraphics[width=0.9\linewidth, height=0.4\linewidth]{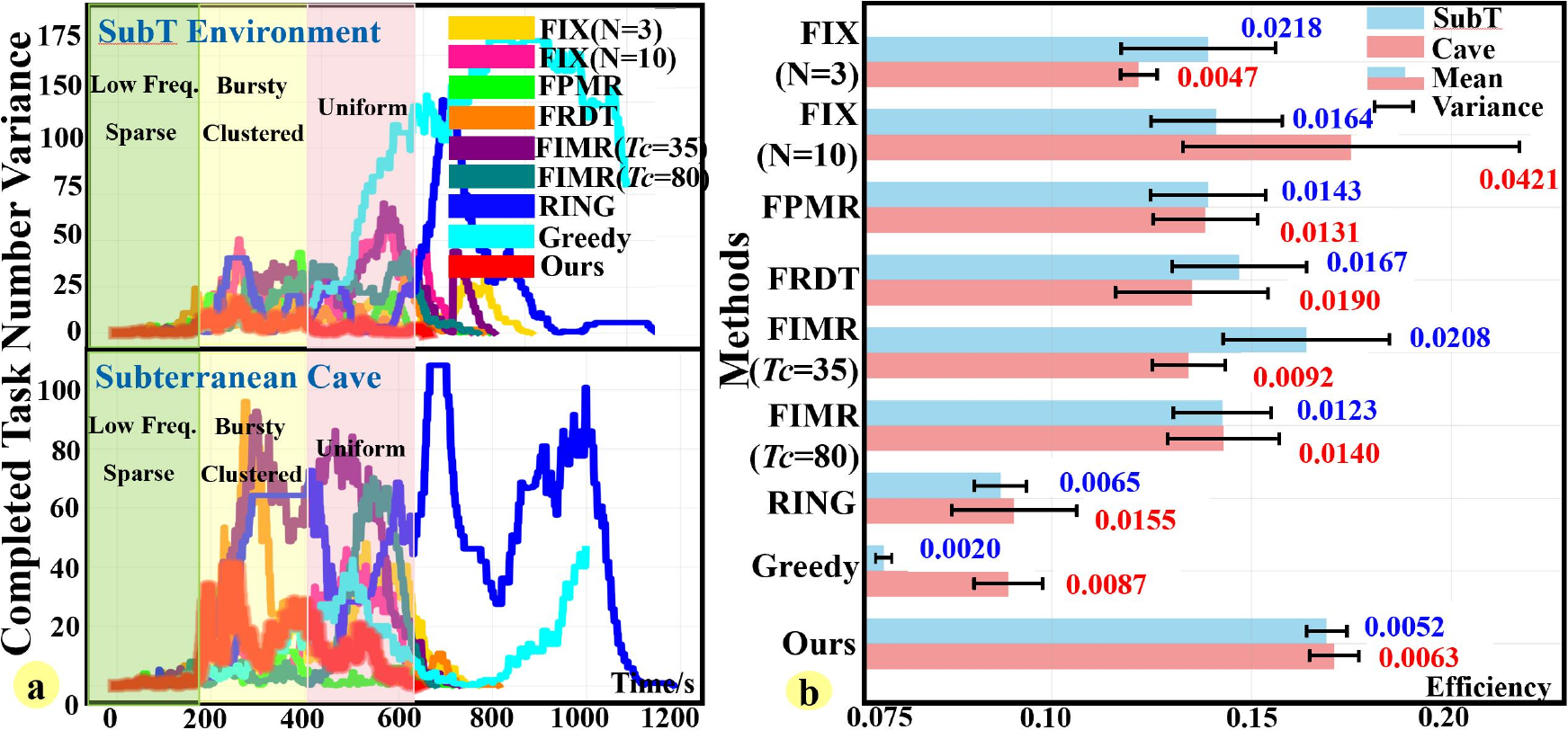}
    \vspace{-3mm}
    \caption{Robustness analysis: instantaneous variance of completed tasks (\textbf{left})
    and task-completion efficiency (\textbf{right}).}
    \label{fig:robust-result}
    \vspace{-1mm}
\end{figure}

\begin{figure}[t!]
    \centering
    \includegraphics[width=0.9\linewidth, height=0.35\linewidth]{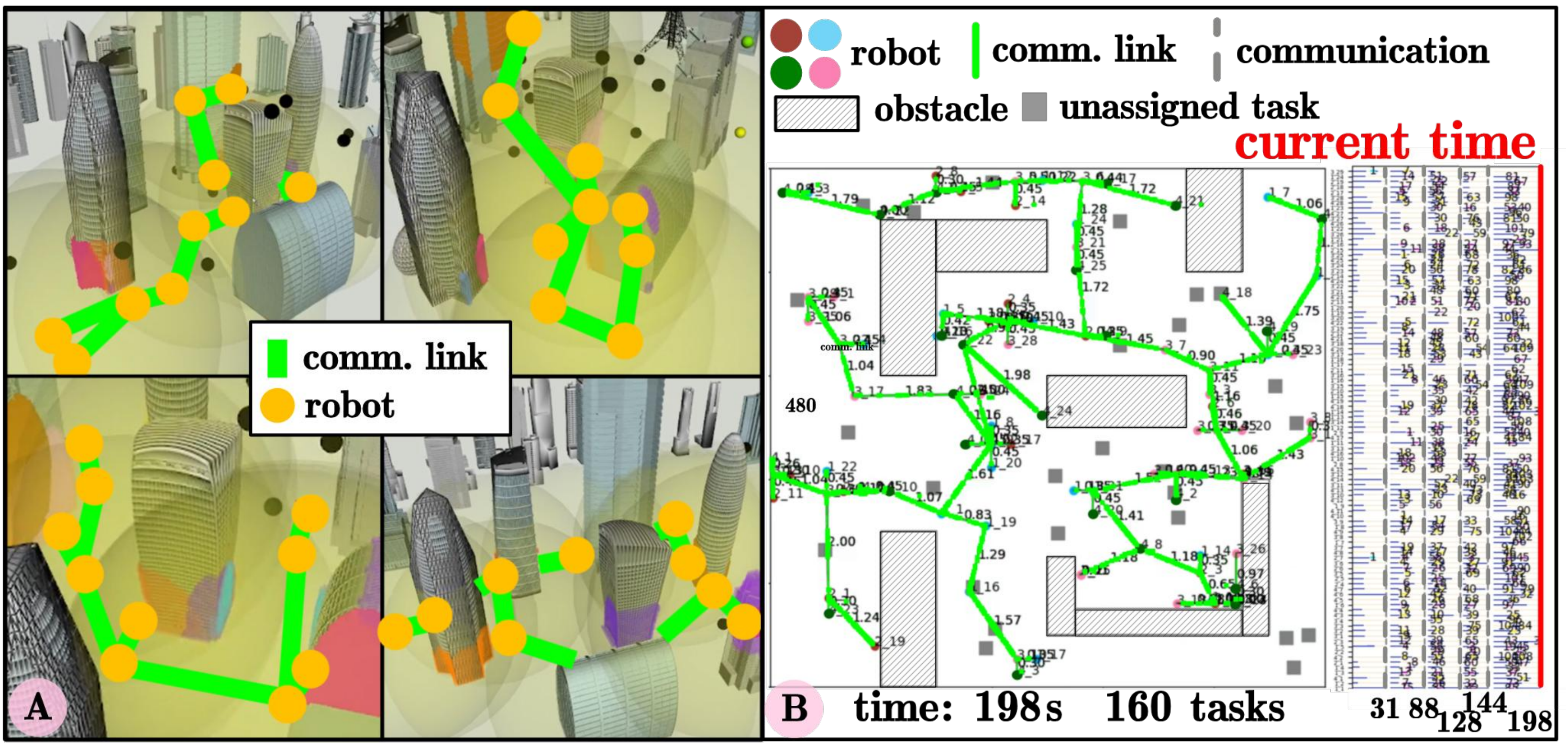}
    \vspace{-2mm}
    \caption{Snapshots of extensive experiments showing operation in a 3D urban environment (\textbf{left})
    and a large-scale deployment with 100 robots (\textbf{right}).}
    \label{fig:extensive}
    \vspace{-1mm}
\end{figure}

\begin{figure*}[!t]
    \centering
    \includegraphics[width=0.99\linewidth,height=0.17\linewidth]{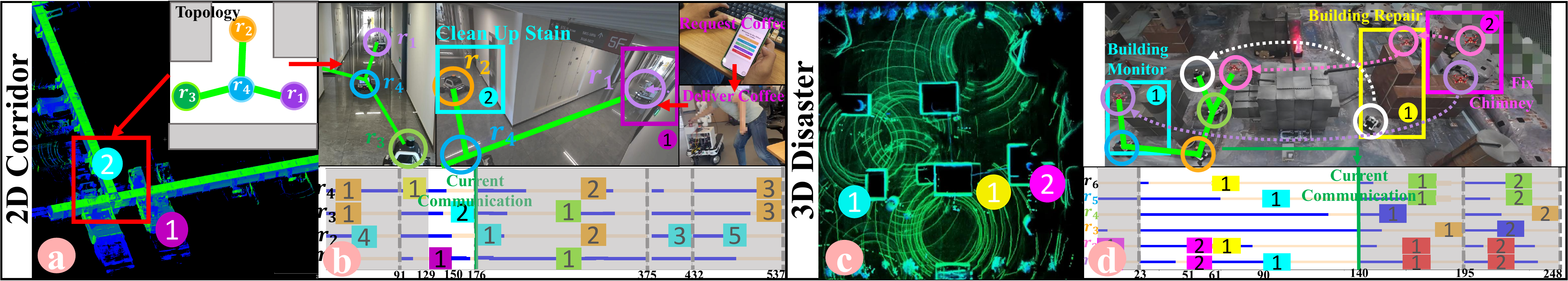}
    \vspace{-2mm}
    \caption{Snapshots of hardware experiments:
      \textbf{(a)} Experimental setup for service robots in office;
      \textbf{(b)} Key snapshots of task execution;
      \textbf{(c)} Hardware deployment in 3D disaster response;
      \textbf{(d)} Key snapshots of collaborative task execution.}
    \label{fig:hardware}
    \vspace{-3mm}
\end{figure*}

\subsubsection{Robustness under Varying Task Distributions}
\label{subsubsec:robust-analysis}

To assess robustness under unknown spatiotemporal task distributions, a 600s
experiment is conducted with three phases. During the first 200s, tasks follow
a spatially sparse and temporally low-frequency pattern. Over $[200\text{s},
400\text{s}]$, they shift to a spatially clustered and temporally bursty pattern.
From 400s to 600s, the distribution becomes spatially and temporally uniform.
Fig.~\ref{fig:robust-result}(a) reports the instantaneous variance of completed
tasks. The proposed method maintains consistently low variance under changing
conditions, whereas baselines exhibit sharp fluctuations.
Fig.~\ref{fig:robust-result}(b) compares task-completion efficiency, defined as
the slope of the task-completion curve. The proposed method sustains high mean
efficiency with low variance in both environments (0.173 $\pm$ 0.005 in SubT and
0.176 $\pm$ 0.006 in caves), while other baselines show substantial volatility.
Greedy yields the lowest mean efficiency, and FIX with $N=10$ exhibits high
variance, underscoring the weakness of fixed thresholds.


\vspace{-1.08em} 

\subsection{Generalizations}
\label{subsec:extensive-experiments}

\subsubsection{3D Environment}
\label{subsec:3d-environment}

To evaluate performance in complex three-dimensional spaces,
an urban environment of $60 \times 40 \times 20\,\text{m}^3$ is designed,
deploying a heterogeneous team of five UAVs and five UGVs.
Task types are designed to leverage aerial capabilities, 
such as high-altitude damage inspection and building fire-rescue missions.
Figure~\ref{fig:extensive} (left) illustrates the resulting communication topologies
that emerge during task execution.
The system successfully completes 45 tasks within 415s while triggering only 11
communication events, highlighting scalability of our method in 3D settings.

\subsubsection{Large-scale Fleets}
\label{subsubsec:scalability}

Scalability is further assessed in a large-scale deployment of $N=100$ robots
performing tasks analogous to those in the SubT environment.
As shown in Fig.~\ref{fig:extensive}, the system completes 160 tasks within 198s,
requiring only six communication events to maintain coordination.
All tasks satisfy their temporal constraints, and the computation times is 5-15s.
These results demonstrate the capability to operate for large-scale robotic fleets.

\subsection{Hardware Experiments}
\label{subsec:hardware-experiments}

To assess practicality, hardware experiments are conducted in office service and
3D disaster-response environments, demonstrating feasibility in real-world settings.

\subsubsection{Service Robots in Office}
The first experiment uses a $138 \times 114 \times 3,\text{m}^3$ office environment
(Fig.~\ref{fig:hardware}) with four UGVs: one delivery, one maintenance, and two
cleaning robots. Each robot is equipped with ad-hoc networking units
(Fig.~\ref{fig:communication-constraints}) and navigation modules. Five task
types are considered with temporal priorities: fault detection (highest), indoor
delivery and garbage collection (medium), and stain removal and equipment
maintenance (lowest). At 80s, a floor stain is detected along with a coffee
delivery request. By 91s, a communication graph is formed and both tasks are
assigned. The delivery robot finishes at 129s (purple block). A cleaning robot
then handles a stain detected at 137s, starting removal at 150s (blue block).
At 176s, a new communication topology is established for subsequent planning.
Overall, 11 tasks are completed within 537s with only 5 communication events,
indicating responsive replanning under dynamic and human-triggered
task arrivals.

\subsubsection{3D Disaster Response}
The second experiment takes place in a $16 \times 14 \times 26,\text{m}^3$ disaster
environment (Fig.~\ref{fig:hardware}). Six task types are considered, including
firefighting before rescue and chimney servicing before building repair, along
with resource collection and water protection. The team includes two UAVs for
monitoring/maintenance and four UGVs: two for firefighting/repair and two for
rescue. The first planning round completes at 23s, after which UAVs execute
chimney repair, completed at 51s (pink block). Maintenance requires UAV--UGV
collaboration at 61s and again at 90s, with communication re-established after
each phase to synchronize actions. In total, 12 tasks are completed in 248s,
demonstrating timely responsiveness and effective ground--air coordination in
emergency conditions.

\vspace{-0.8em}
\section{Conclusion \& Future work} \label{sec:conclusion}

This work presents a branch-and-bound framework that co-optimizes task collaboration 
and team-wise communication online, 
for multi-agent systems operating in dynamic and unknown environments.
The proposed approach achieves efficient information exchange
while meeting temporal task requirements, as validated through extensive simulations
and hardware experiments.
Future research will explore
the hierarchical combination of global and local coordination protocols.


\vspace{-0.8em}
\bibliographystyle{IEEEtran}
\bibliography{contents/arxiv}

\end{document}